\documentclass[fleqn,10pt]{wlscirep}
\usepackage[utf8]{inputenc}
\usepackage[T1]{fontenc}
\usepackage{times}
\usepackage{soul}
\usepackage{url}
\usepackage{hyperref}
\usepackage{caption}
\usepackage{graphicx}
\usepackage{amsmath}
\usepackage{amsthm}
\usepackage{booktabs}
\usepackage{algorithm}
\usepackage{algorithmic}
\usepackage{enumerate}
\setlist{itemsep=0pt, parsep=0pt}
\usepackage[switch]{lineno}
\usepackage{amsfonts,etoolbox}
\usepackage{xcolor}
\usepackage{subcaption}
\usepackage{array}

\newtheorem{theorem}{Theorem}

\newtheorem{assumption}{Assumption}

\definecolor{mygreen}{RGB}{101, 163, 62}
\definecolor{myblue}{RGB}{60, 103, 188}
\definecolor{mygray}{RGB}{79, 79, 79}
\definecolor{myorange}{RGB}{234, 114, 43}

\title{AsynDBT: Asynchronous Distributed Bilevel Tuning for efficient In-Context Learning with Large Language Models}

\author[1]{Hui Ma}
\author[2]{Shaoyu Dou\thanks{For Shaoyu Dou: This work was conducted at Tongji University and is unrelated to the author's current affiliation with Ant Group.}}
\author[3,$\dagger$ ]{Ya Liu}
\author[4,$\dagger$]{Fei Xing}
\author[5]{Li Feng}
\author[6]{Feng Pi} 
\affil[1]{Xinjiang Key Laboratory of Intelligent Computing and Smart Applications, School of Software,
Xinjiang University, Urumqi, 830091, China}
\affil[2]{the Department of Computer Science  and Technology, Tongji University, Shanghai, 201800, China}
\affil[3]{School of Information Science and Engineering, Zaozhuang University, Zaozhuang, 277160, China}
\affil[4]{Xinjiang University, College of Geography and Remote Sensing Sciences, Urumqi, 830046, China}
\affil[5]{the Hochschule Bielefeld-University of Applied Sciences and Arts, Bielefeld, Germany}
\affil[6]{Xinjiang General Station of Exit and Entry Frontier Inspection, 83000l, Urumqi, China}
\affil[$\dagger$]{Corresponding author:  liuya@uzz.edu.cn, xingfei@xju.edu.cn}


\begin{abstract}
With the rapid development of large language models (LLMs), an increasing number of applications leverage cloud-based LLM APIs to reduce usage costs. However, since cloud-based models' parameters and gradients are agnostic, users have to manually or use heuristic algorithms to adjust prompts for intervening LLM outputs, which requiring costly optimization procedures. In-context learning (ICL) has recently emerged as a promising paradigm that enables LLMs to adapt to new tasks using examples provided within the input, eliminating the need for parameter updates. Nevertheless, the advancement of ICL is often hindered by the lack of high-quality data, which is often sensitive and different to share. Federated learning (FL) offers a potential solution by enabling collaborative training of distributed LLMs while preserving data privacy. Despite this issues, previous FL approaches that incorporate ICL have struggled with severe straggler problems and challenges associated with heterogeneous non-identically data. To address these problems, we propose an asynchronous distributed bilevel tuning (AsynDBT) algorithm that optimizes both in-context learning samples and prompt fragments based on the feedback from the LLM, thereby enhancing downstream task performance. Benefiting from its distributed architecture, AsynDBT provides privacy protection and adaptability to heterogeneous computing environments. Furthermore, we present a theoretical analysis establishing the convergence guarantees of the proposed algorithm. Extensive experiments conducted on multiple benchmark datasets demonstrate the effectiveness and efficiency of AsynDBT.
\end{abstract}
\begin{document}

\flushbottom
\maketitle
%
%
\thispagestyle{empty}


\section*{Introduction}
The emergence of large language models (LLMs) has introduced a revolutionary solution to meet the growing demand for advanced intelligent services\cite{zhou20241111}. Unlike traditional deep neural networks\cite{10330754,9055742}, LLMs are trained on massive and diverse datasets and contain billions of parameters. This enables them not only to perform logical reasoning and complex thinking but also to excel across a wide range of natural language processing tasks, including text generation, classification, machine translation, and question answering\cite{liu2025argus,liu2025asynchronous}. As the scale of LLMs has grown to encompass hundreds of billions of parameters, local deployment has become increasingly impractical. Consequently, cloud-based deployment accessed via application programming interfaces (APIs) has gradually become the dominant approach, providing users with  efficient and flexible interaction capabilities\cite{ding2024222}. Nevertheless, the training and maintenance costs of LLMs remain extremely high. Therefore, designing and implementing LLMs that achieve both efficient training and strong generalization performance across diverse downstream tasks continues to pose a critical research challenge.

A practical solution for improving LLM performance on downstream tasks is in-context learning (ICL)\cite{dong2023survey}. Unlike conventional approaches that require fine-tuning or updating model weights, ICL enables models to learn and generalize to novel tasks by providing labeled context examples (demonstrations) alongside instruction prompts. Moreover, although prior studies have shown that in-context samples can influence LLM performance on downstream tasks \cite{min2022rethinking}, the underlying mechanisms remain largely unexplored. Recent studies \cite{dai-etal-2023-gpt,Shiguang2025,10.24963/ijcai.2024/702} suggest that ICL can be interpreted from a meta-learning perspective, in which pre-trained LLMs act as meta-optimizers and in-context samples are used to compute the meta-gradient\cite{genewein2025unde}. Motivated by this, we posit a hierarchical relationship between demonstration selection and instruction prompt editing. Nevertheless, systematic research on applying bilevel optimization to model and analyze the in-context learning process remains in its infancy.

The downstream task performance of LLM agents based on ICL strongly depends on the quality of the provided demonstrations. In real-world applications, however, obtaining high-quality demonstration samples poses numerous challenges, including the need for domain expertise, the high cost of manual annotation, data privacy concerns, client-specific restrictions, and the limited availability of suitable examples. These challenges significantly limit the potential of ICL methods to enhance LLMs performance. To mitigate these issues, federated learning (FL) has emerged as a promising solution\cite{9912240}, enabling collaborative model training on a central server without direct access to client data. Recently, several studies have explored integrating FL with ICL \cite{Ruhan2025icml,Wu2024arxiv}. Nonetheless, most of these methods rely on synchronous distributed algorithms, which often suffer from the “straggler problem,” thereby impeding scalability and overall efficiency.

Furthermore, in the context of FL with ICL, it is particularly vulnerable to malicious participants who may inject low-quality or adversarial samples into the pool of in-context learning examples. This form of data poisoning poses a significant risk of compromising the global optimization process. To address such threats, recent studies have introduced efficient federated learning frameworks designed to mitigate the impact of malicious workers while providing theoretical robustness guarantees\cite{10852413,3534678.3539231,ma2025cellular}. For instance, Jiang et al.\cite{10852413} proposed a global model recovery method utilizing selective information storage and adaptive model rollback. Similarly, our prior work\cite{ma2025cellular} demonstrated the efficacy of regularization techniques in enhancing the robustness of deep learning models. Building on these foundations, a critical challenge addressed in this study is the design and implementation of robust aggregation techniques tailored to enhance the model's resilience against data poisoning attacks within the FL with ICL paradigm.


To address the aforementioned challenges, this study proposes an asynchronous bilevel ICL framework based on federated learning, which comprises multiple workers holding private data and a central parameter server that coordinates their collaboration. \textit{Then}, we formulate ICL as a bilevel programming problem, where the upper-level objective focuses on the selection of in-context samples, and the lower-level objective adjusts fragments of the prompt. \textit{Besides}, we design an asynchronous distributed algorithm that incorporates the lower-level optimization as a constraint on the upper-level optimization. By approximating the feasible region of the lower-level problem with polyhedral constraints, the bilevel formulation is transformed into a single-level optimization problem, thereby facilitating efficient distributed computation. \textit{Furthermore}, we incorporate a regularization-based robust optimization mechanism to enhance the model's resilience against data poison attacks.


 The main contributions of this study can be summarized  as follows:
\begin{itemize}
    \item  We propose a novel bilevel black-box learning framework that captures the hierarchical nature of prompt tuning via integrating prompt editing with in-context sample selection. To the best of our knowledge, this work represents the first attempt to establish a bilevel in-context learning framework.
    

    \item We design an asynchronous distributed bilevel tuning algorithm, \textbf{AsynDBT}, to address the distributed bilevel black-box optimization problem. As far as we know, it is the first approach to consider the impact of device heterogeneity and the network straggler issue within the federated in-context learning framework. We also provide a theoretical analysis of the convergence properties of AsynDBT.

    \item We validate the effectiveness of the proposed AsynDBT through extensive experiments on multiple benchmark datasets, demonstrating its superiority over several state-of-the-art prompt tuning methods in terms of downstream task performance and computational efficiency. Furthermore, we highlight its practical applicability in 5G AIOps, where our method achieves an accuracy improvement of nearly 10\% compared with baseline models.
\end{itemize}

\section*{Related Work}


\subsection*{Prompt Tuning}
For the black-box discrete prompt tuning problem explored in this paper, there are two main categories of approaches. The first category relies on the reasoning capabilities of LLMs, which are directly used to summarize the reasons for wrong reasoning and tuning the prompt. OPRO \cite{yang2023large} develops two LLM-based agents: one for scoring the optimized prompt fragment and another for generating a new prompt based on historical scores. The tuning of the prompt fragment involves multiple interactions between the above two agents. The APO \cite{pryzant2023automatic} employs semantic gradient descent for prompt tuning. it analyzes the reasons for incorrect answers by querying the LLM, treating the obtained analysis as the ``semantic gradient''. Then, LLM generates new prompts based on the opposite of the semantic gradient and selects the most effective one through bandit selection. Like APO, \cite{sun2023autohint} queries LLM to summarize the reasons for incorrect responses, and then adds the generated summary to the original prompt.

The second category employs parametric models to generate prompts. The parameters are learned based on the feedback of LLM. AutomateCoT \cite{shum2023automatic}, a two-stage method for generating chain-of-thought (CoT) for specific tasks. Initially, it creates a CoT pool using manually written or LLM-generated CoTs, then evaluates each candidate's reasoning results using the LLM. Following this, it filters out ineffective CoTs and optimizes a categorical distribution parameter to select the most helpful and suitable CoT. Similarly, BDPL \cite{diao2022black} optimizes the parameters of a categorical distribution for generating a prompt fragment attached to the original prompt. RLprompt \cite{deng2022rlprompt} employs an LLM with fixed parameters and a trainable MLP as a policy network, and optimizes its parameters based on the feedback from the black-box LLM. However, prompt tuning methods often suffer from overfitting, as they require updating a large number of parameters with limited data.

\subsection*{Demonstration Selection}
In generally, demonstration selection methods generally fall into two categories: supervised and unsupervised approaches\cite{dong2023survey}.

Supervised demonstration selection methods aim to retrieve entire demonstration sets to model inter-relationships between examples. For instance, Li et al.\cite{Xiaonan2023} introduced a unified retriever that selects demonstrations across different tasks. Mahankali et al. \cite{mahankali2024one} proposed a model-adaptive method that uses LLMs to predict unlabeled data, assigning an uncertainty score to each instance. By contrast, unsupervised demonstration selection algorithms are heuristic, or directly generate demonstrations by identififying the nearest neighbors of input instances. 
For example, Liu et al. \cite{liu2022makes} employed a pre-trained encoder to embed both test samples and candidate demonstrations as vectors, then selects the K nearest neighbors as in-context examples for each test instance. Levy et al. \cite{levy2022diverse} proposed Cover-LS, which removes one of two demonstrations if they are semantically similar. In addition, using output scores of LLMs as unsupervised metrics has shown effectiveness in demonstration selection\cite{li-qiu-2023-finding,wu2023111}. Particularly, Kim et al. \cite{kim2022self} generates in-context samples directly with LLMs, while Wu et al. \cite{wu2023111} selected the best subset permutation of kNN examples based on the code length for data transmission.

Most investigated methods treat prompt tuning and in-context sample selection as separate processes, seldom optimizing them jointly. AutoCoT \cite{zhang2022automatic} is the only method that heuristically implements joint optimization. Whereas, we recognize the natural hierarchical structure between the prompt and the in-context samples, which motivates us to model the ICL problem as bilevel programming and propose an asynchronous distributed optimization algorithm, AsynDBT. To the best of our knowledge, AsynDBT is the first distributed algorithm that jointly considers both prompt editing and in-context sample selection with theoretical convergence guarantees.

\section*{Distributed Bilevel In-context Learning}
\subsection*{Problem Formulation }
A prompt equipped with ICL can be represented as the combination of ICL content and a test query, i.e. $[f_{t1}($\textcolor{myorange}{$\mathbf{S}$}$,[$\textcolor{mygray}{$\mathbf{I}$}$;$\textcolor{mygreen}{$\mathbf{T}$}$]); f_{t2}($\textcolor{myblue}{$\mathbf{X}_q$}$,[$\textcolor{mygray}{$\mathbf{I}$}$;$\textcolor{mygreen}{$\mathbf{T}$}$])$, where $\mathbf{I}$ is the task instruction, $\mathbf{T}$ is a fragment appended to the task instruction to be optimized, $\mathbf{S}$ represents a set of in-context samples, $\mathbf{X}_q$ is the query sample. $f_{t1}(\cdot,\cdot)$ and $f_{t2}(\cdot,\cdot)$ are used to concatenate their input into ICL content and query content respectively. $[\cdot;\cdot]$ denotes the concatenation. We provide the following example with colors corresponding to the parameters in the formulas.

\scalebox{0.85}{
\hspace{-15pt}
\begin{minipage}{\textwidth}
\vspace{12pt}
\textcolor{mygray}{Question: Is the sentiment of ``}\textcolor{myorange}{I am happy}\textcolor{mygray}{'' positive or negative?} \textcolor{mygreen}{Note: ...}. \\
\textcolor{mygray}{Answer:} \textcolor{myorange}{positive} \\
\textcolor{mygray}{Question: Is the sentiment of ``}\textcolor{myorange}{I am angry}\textcolor{mygray}{'' positive or negative?} \textcolor{mygreen}{Note: ...}. \\
\textcolor{mygray}{Answer:} \textcolor{myorange}{negative} \\
\textcolor{mygray}{Question: Is the sentiment of ``}\textcolor{myblue}{I am hungry}\textcolor{mygray}{'' positive or negative?} \textcolor{mygreen}{Note: ...}. \\
\textcolor{mygray}{Answer:} \\
\vspace{6pt}
\end{minipage}
}

In this paper, we take the $U$ class classification task \cite{sun2022black} as an example and denote the training set with labels as $\mathcal{S} = \mathcal{S}_1 \cup \cdots \cup \mathcal{S}_U \cup \mathcal{S}_Q$, where $\mathcal{S}_u = \{s_j = (\mathbf{X}_ j, y_j);j=1,\cdots,|\mathcal{S}_u| \}$, representing the training set of class $u, u=1, \cdots, U$. The number of samples in each category is $V$, i.e. $|\mathcal{S}_u| = V$.
In general, $\mathbf{I}$ is fixed, $\mathbf{X}_q$ is sampled from a fixed set $\mathcal{S}_Q$. We intervene in the output of LLM by adjusting $\mathbf{T}$ and $\mathbf{S}$. The following assumptions are made in this paper:
\begin{itemize}
    \item Let $\mathbf{T}=[t_1; \cdots; t_M]=\left[\mathcal{V}[j_1]; \cdots; \mathcal{V}[j_M] \right]$ , where $\mathcal{V}$ is a fixed vocabulary set \cite{diao2022black}, $|\mathcal{V}|=N$, and $\mathcal{V}[j_i]$ denotes the $j_i$th word in $\mathcal{V}$. We assume that each word in $\mathbf{T}$ is independent of each other and sampled from different categorical distributions, i.e. $j_i \sim Cat(\mathbf{p}_i), \mathbf{p}_i \in \mathbb{R}^{N}$. Correspondingly, we denote the random variable obeying the distribution $Cat(\mathbf{p}_i)$ as $J_i$.

    \item We select one sample from each category of labeled samples as in-context samples. Let $\mathbf{S} = \{s_1, \cdots, s_U\} = \{\mathcal{S}_1[k_1], ... , \mathcal{S}_U[k_U]\}$, and each demonstration is independently sampled from $\mathcal{S}_1, \cdots, \mathcal{S}_U$ , i.e. $k_i \sim Cat(\mathbf{q}_i), \mathbf{q}_i \in \mathbb{R }^V$. Denote the random variable obeying the distribution $Cat(\mathbf{q}_i)$ as $K_i$.
\end{itemize}
To simplify the subsequent equations, the output of the black-box LLM $f_{LLM}([f_{t1}(\mathbf{S},[\mathbf{I};\mathbf{T}]),f_{t2}(\mathbf{X}_q,[\mathbf{I};\mathbf{T}])])$ is abbreviated as $f_{LLM}(\mathbf{T}, \mathbf{S}, \mathbf{X}_q)$. When considering the classification task, the cross-entropy loss can be written as,
\begin{equation}
\mathcal{L}_{ce}(\mathbf{T}, \mathbf{S})  = \sum_{(\mathbf{X}_q,y_q) \in \mathcal{S}_Q} L_{ce}(f_{LLM}([\mathbf{T},\mathbf{S}, \mathbf{X}_q]), y_q)  = \mathcal{L}_{ce}(j_1, \cdots, j_M, k_1, \cdots, k_U).
\end{equation}

Based on the above assumptions, prompt tuning in the \textit{centralized} black-box setting can be formulated as the following bilevel optimization problem,
\begin{equation}
\label{bilevel1}
\begin{aligned}
\min \ \ &  \mathbb{E}_{K_1,\cdots,K_U}[ \mathcal{L}_{ce}(\mathbf{p}_i^{\star} (K_1, \cdots, K_U) ; J_1, \cdots, J_M,K_1, \cdots, K_U)], \\
\text{s.t.} \ \ & \mathbf{p}_i^{\star} (K_1, \cdots, K_U) = \arg\min_{\mathbf{p}_i} \ \mathbb{E}_{J_i}[\mathcal{L}_{ce}(\mathbf{p}_1,\cdots,\mathbf{p}_M; J_1, \cdots, J_M,K_1, \cdots, K_U)],\\
& \mathbf{1}^{\top}\mathbf{p}_i = 1 ,\ \mathbf{1}^{\top}\mathbf{q}_{i'} = 1, \\
& 0 \leq p_{i,j} \leq 1 ,\ 0 \leq q_{i',j'} \leq 1, \\
&i=1,\cdots, M, \ j = 1, \cdots, N, \ i' = 1, \cdots, U, \ j' = 1, \cdots, V,\\
\text{var.} \ \ & \mathbf{p}_1, \cdots, \mathbf{p}_M, \mathbf{q}_{1},\cdots,\mathbf{q}_{U}, 
\end{aligned}
\end{equation}
where $p_{i,j}$ is the $j$th item of $\mathbf{p}_i$, $q_{i',j'}$ is the $j'$th item of $\mathbf{q}_{i'}$. For a given random variable $K_i, \forall \ i \in (1,\cdots,U)$, sampled from the distribution parameterized by $(\mathbf{q}_1,\cdots,\mathbf{q}_U)$, the lower-level problem optimizes $\mathbf{p}_i$ while treating this realization as fixed. We denote the resulting objective by $g(\mathbf{p}_i; K_1, \cdots, K_U)$, and when the dependency on $K$ is clear, we abbreviate it as $g(\mathbf{p}_i)$.


\subsection*{ICL with Federated Learning Framework}
In this paper, we consider a federated scenario with a parameter server and $R$ heterogeneous workers\cite{assran2020advances}, which are composed of $N_w$ benign workers and $B$ malicious workers, i.e., $R=N_w+B$. The malicious workers, as described in \cite{Zhu2022}, will collude with each other and send arbitrary malicious messages to the server. Additionally, the identity of malicious workers is a priori unknown to the server. Despite the presence of $B$ malicious workers, our objective is to effectively leverage $N_w$ benign local workers for distributed training and obtain a robust global model, where each worker has its own training set as well as corresponding optimization variables. For worker $v$, denote its $i$th lower-level optimization variable as $\mathbf{p}_i^{(v)}$ (i.e., the distribution parameter of the $i$th word in a fragment) and its $i'$th upper-level optimization variable as $\mathbf{q}_{i'}^{(v)}$ (i.e., the distribution parameter of the $i'$th demonstration). 
We denote all lower and upper-level optimization variables of worker $v$ as $\{\mathbf{p}_i^{(v)}\}$ and $\{ \mathbf{q}_{i'}^{(v)}\}$, respectively, and denote all lower and upper-level variables as $\{\mathbf{p}_i\}$ and $\{ \mathbf{q}_{i'}\}$.

Since the distribution of the training dataset varies across workers, and each worker adjusts their in-context samples based only on the feedback from LLM on their local dataset, thus the upper-level optimization variables are not needed to enforce a global consensus. So we only introduce a set of consensus variables $\{\mathbf{z}_i\}$ for the lower-layer variables \cite{matamoros2017asynchronous}. Suppose $v = 1, \cdots, N_w$, the formulation of the prompt tuning in the \textit{distributed} black-box setting is given below,
\begin{equation}
    \label{distri_opt}
    \begin{array}{rl}
    \min & F(\{\mathbf{p}_i^{\star}\},\{\mathbf{z}_i^{\star}\},\{\mathbf{q}_{i'}\}) = \sum_{v=1}^{N_w} f(\mathbf{p}_1^{(v)\star}, \cdots, \mathbf{p}_M^{(v)\star}, \mathbf{q}_1^{(v)}, \cdots, \mathbf{q}_U^{(v)}) \\
    \text{s.t.} & \{\mathbf{p}_i^{(v)\star}\}, \mathbf{z}_i^{\star} = \arg\min_{\{\mathbf{p}_i^{(v)}\}, \mathbf{z}_i} G(\{\mathbf{p}_i^{(v)}\}, \mathbf{z}_i)   = \sum_{v=1}^{N_w} g(\mathbf{p}_i^{(v)}) \\
    &  \mathbf{z}_i = \mathbf{p}_i^{(v)}\\
    & \mathbf{1}^{\top}\mathbf{p}_i^{(v)} = 1 ,\ \mathbf{1}^{\top}\mathbf{q}_{i'}^{(v)} = 1 \\
    & 0 \leq p_{i,j}^{(v)} \leq 1 ,\ 0 \leq q_{i',j'}^{(v)} \leq 1 \\
    \text{var.} & \{\mathbf{p}_i^{(v)}\}, \ \{\mathbf{q}_{i'}^{(v)}\}, \mathbf{z}_i .
    \end{array}
\end{equation}

\section*{Optimization of Distributed Black-box Bilevel Problem}
For the $i$th lower-level variable, we define an estimation function $\phi_i (\cdot)$ for the solution, that is,
\begin{equation}
    \label{phi_function}
    \begin{array}{ll}
    \phi_i (k_1, \cdots, k_U)  \underset{\{\mathbf{p}_i^{(v)}\}, \mathbf{z}_i}{\arg\min} \ \left\{ \sum_{v=1}^{N_w} g(\mathbf{p}_i^{(v)}) : \mathbf{z}_i = \mathbf{p}_i^{(v)}, v = 1, \cdots, N_w \right\},
    \end{array}
\end{equation}
where the arguments in $\phi_i (\cdot)$ derive from the full form of $g(\mathbf{p}_i^{(v)})$, that is, $g(\mathbf{p}_i^{(v)}, k_1, \cdots, k_U)$. We also define
\begin{equation}
    \label{h_function}
    h_i (k_1, \cdots, k_U, \{\mathbf{p}_i^{(v)}\}, \mathbf{z}_i) = \| \begin{bmatrix}  \{\mathbf{p}_i^{(v)}\} \\  \mathbf{z}_i \end{bmatrix} - \phi(k_1, \cdots, k_U) \|_1.
\end{equation}
Then Eq. $\eqref{distri_opt}$ can be rewritten by $h_i (\cdot)$ as
\begin{equation}
    \label{distri_opt2}
    \begin{array}{rl}
    \min & F(\{\mathbf{p}_i\},\{\mathbf{z}_i\},\{\mathbf{q}_{i'}\}) \\
    \text{s.t.} & h_i(k_1, \cdots, k_U, \{\mathbf{p}_i^{(v)}\}, \mathbf{z}_i) = 0\\
    & \mathbf{1}^{\top}\mathbf{p}_i^{(v)} = 1 ,\ \mathbf{1}^{\top}\mathbf{q}_{i'}^{(v)} = 1 \\
    & 0 \leq p_{i,j}^{(v)} \leq 1 ,\ 0 \leq q_{i',j'}^{(v)} \leq 1 \\
    \text{var.} & \{\mathbf{p}_i^{(v)}\}, \ \{\mathbf{q}_{i'}^{(v)}\}, \mathbf{z}_i.
    \end{array}
\end{equation}

\subsection*{Estimation of the $i$th Lower-level Variables} 
Although the optimization problem Eq. $\eqref{distri_opt}$ requires the exact solution of Eq. $\eqref{phi_function}$, previous research \cite{li2022local,yang2021provably} has demonstrated that it is sufficient to use an approximation of Eq. $\eqref{phi_function}$ with bilevel optimization. Therefore, in our approach, we estimate the solution using $K$-step gradient descent. Taking into account the consensus constraint in Eq. $\eqref{phi_function}$, the corresponding Lagrangian function of Eq. $\eqref{phi_function}$ can be formulated as follows,
\begin{equation}
    \label{gp}
    \begin{array}{ll}
    G_{Pi}(\{\mathbf{p}_i^{(v)}\}, \mathbf{z}_i, \{\boldsymbol{\rho}_i^{(v)}\}) 
    = \sum_{v=1}^{N_w} \left(g(\mathbf{p}_i^{(v)}) + \boldsymbol{\rho}_i^{(v)\top}(\mathbf{p}_i^{(v)} - \mathbf{z}_i) + \frac{\mu}{2} \| \mathbf{p}_i^{(v)} - \mathbf{z}_i \|_2^2\right),
    \end{array}
\end{equation}
where $\boldsymbol{\rho}_i^{(v)} \in \mathbb{R}^{N}$ is the Lagrangian dual variable and $\mu > 0$ is the penalty parameter. The procedure for computing the approximate solution of Eq. $\eqref{phi_function}$ using $K$ steps gradient descent is as follows.

\textit{For worker $v$}: In the $k+1$th iteration, worker $v$ first updates their local lower-level optimization variables on the local data. 
\begin{equation}
\label{inner_p}
\mathbf{p}_i^{(v) k+1} = proj_{\mathcal{P}} \left( \mathbf{p}_i^{(v) k} -  \eta_{\mathbf{p}_i^{(v)}} \left( \nabla_{\mathbf{p}_i^{(v)}} G_{Pi}(\{\mathbf{p}_i^{(v) k}\}, \mathbf{z}^{k}_i, \boldsymbol{\rho}_i^{(v) k} ) + \psi \text{sign}(\mathbf{z}_i^{k}-\mathbf{p}_i^{t}) \right) \right),
\end{equation}
where $\psi$ represents a positive constant. Subsequently, each worker sends its variables $\mathbf{p}_i^{(v) k+1}, \mathbf{p}_i^{(v) k}$ to the parameter server. 

\textit{For the parameter server}: It updates the consensus and dual variables according to the following equations.
\begin{small}
\begin{align}
\label{inner_z}
    &\mathbf{z}_i^{k+1} = \operatorname{proj}_{\mathcal{P}} \left( \mathbf{z}_i^{k} - \eta_{\mathbf{z}_i} \left( \nabla_{\mathbf{z}_i} G_{Pi}(\{\mathbf{p}_i^{(v) k+1}\}, \mathbf{z}^{k}_i, \boldsymbol{\rho}_i^{(v) k} ) + \psi \big( \sum_{i \in \{1,...,N_w\}} \text{sign}(\mathbf{z}^{k}-\mathbf{p}_i^{k+1}) + \sum_{j \in \{1,...,B\}} \text{sign}(\mathbf{z}^{k}-\mathbf{p}_j^{k+1})\big) \right) \right), \\
\label{inner_rho}
    &\boldsymbol{\rho}_i^{(v) k+1} = \boldsymbol{\rho}_i^{(v)k} + \eta_{\boldsymbol{\rho}_i^{(v)}} \nabla_{\boldsymbol{\rho}_i^{(v)}} G_{Pi}(\{\mathbf{p}_i^{(v) k+1}\}, \mathbf{z}^{k+1}_i, \boldsymbol{\rho}_i^{(v) k}),
\end{align}
\end{small}
where the projection $\mathbf{x}' = proj_{\mathcal{P}}(\mathbf{x})$ guarantees that the parameter $\mathbf{z}'$ is a legitimate distribution parameter. According to \cite{diao2022black}, $\mathbf{x}' = \min(1, \max(0, \mathbf{x} - v^{\star} \mathbf{1}))$ where $v^{\star} = \arg\min_v \left( \mathbf{1}^{\top} \min( 1, \max(0, \mathbf{x}-v\mathbf{1}))-1 \right)$. Thus, after $K$ times of communication, the estimation of the $i$th lower-level problem is $\phi_i (k_1, \cdots ,k_U) = [\{ \mathbf{p}_i^{(v)K} \}, \mathbf{z}_i^{K}]^\top$.
Specifically, the gradient of Eq. $\eqref{gp}$ with respect to each argument is as follows:
\begin{align}
    \nabla_{\mathbf{p}_i^{(v)}} G_{Pi}(\{\mathbf{p}_i^{(v)}\}, \mathbf{z}_i, \boldsymbol{\rho}_i^{(v)}) 
    &= \nabla_{\mathbf{p}_i^{(v)}} g_v(\mathbf{p}_i^{(v)}) + \boldsymbol{\rho}_i^{(v)} + \mu (\mathbf{p}_i^{(v)} - \mathbf{z}_i), \\
    \nabla_{\mathbf{z}_i} G_{Pi}(\{\mathbf{p}_i^{(v)}\}, \mathbf{z}_i, \boldsymbol{\rho}_i^{(v)} ) 
    &= \sum_{v=1}^{N_w} \left(-\boldsymbol{\rho}_i^{(v)} + \mu (\mathbf{z}_i - \mathbf{p}_i^{(v)}) \right), \\
    \nabla_{\boldsymbol{\rho}_i^{(v)}} G_{Pi}(\{\mathbf{p}_i^{(v)}\}, \mathbf{z}_i, \boldsymbol{\rho}_i^{(v)}) 
    &= \mathbf{p}_i^{(v)} - \mathbf{z}_i .
\end{align}

\subsection*{Approximate the Feasible Region of Constraint with Cutting-plane}

After estimating the solution of the lower-level optimization problem, Eq. $\eqref{distri_opt2}$ can be rewritten as,
\begin{equation}
    \label{distri_opt3}
    \begin{array}{rl}
    \min & F(\{\mathbf{p}_i\},\{\mathbf{z}_i\},\{\mathbf{q}_{i'}\}) \\
    \text{s.t.} & h_i (k_1, \cdots, k_U, \{\mathbf{p}_i^{(v)}\}, \mathbf{z}_i) \leq \epsilon \\
    & \mathbf{1}^{\top}\mathbf{p}_i^{(v)} = 1 ,\ \mathbf{1}^{\top}\mathbf{q}_{i'}^{(v)} = 1 \\
    & 0 \leq p_{i,j}^{(v)} \leq 1 ,\ 0 \leq q_{i',j'}^{(v)} \leq 1\\
    \text{var.} & \{\mathbf{p}_i^{(v)}\}, \ \{\mathbf{q}_{i'}^{(v)}\}, \mathbf{z}_i,
    \end{array}
\end{equation}
where $\epsilon>0$ is a constant. $h_i (k_1, \cdots, k_U, \{\mathbf{p}_i^{(v)}\}, \mathbf{z}_i)$ is a convex function with respect to the $\{\mathbf{p}_i^{(v)}\}$ and $\mathbf{z}_i$. According to Eq. $\eqref{h_function}$, its feasible region can be approximated by a polyhedron enclosed by $L$ cutting planes \cite{boyd2007localization,michalka2013cutting,franc2011cutting}. Noting that the feasible region of $h_i(k_1, \cdots, k_U, \{\mathbf{p}_i^{(v)}\}, \mathbf{z}_i)$ is $\mathcal{P}_i$, and let $\mathcal{P}_i^t$ represent polyhedron at iteration $t$. Then,
\begin{equation}
    \mathcal{P}_i^t = \{\mathbf{p}_i \in \mathbb{R}^{N} | \sum_{v=1}^{N_w} \mathbf{a}_i^{l
 (v)\top}\mathbf{p}_i^{(v)} + \mathbf{b}_i^{l\top} \mathbf{z}_i + c_i^l \leq 0; l=1:|\mathcal{P}_i^t| \},
\end{equation}
where $\mathbf{a}_i^{l(v)}, \mathbf{b}_i^l \in \mathbb{R}^{N}$. Using the above linear constraints in place of the constraints in $\eqref{distri_opt3}$, Eq. $\eqref{distri_opt2}$ can be reformulated as,
\begin{equation}
\label{distri_opt4}
\begin{array}{rl}
\min &  F(\{\mathbf{p}_i\},\{\mathbf{z}_i\},\{\mathbf{q}_{i'}\}) \\
\text{s.t.} & \sum_{v=1}^{N_w} \mathbf{a}_i^{l(v)\top}\mathbf{p}_i^{(v)} + \mathbf{b}_i^{l\top} \mathbf{z}_i + c_i^l \leq 0 \\
& \mathbf{1}^{\top}\mathbf{p}_i^{(v)} = 1 ,\ \mathbf{1}^{\top}\mathbf{q}_{i'}^{(v)} = 1 \\
& 0 \leq p_{i,j}^{(v)} \leq 1 ,\ 0 \leq q_{i',j'}^{(v)} \leq 1 \\
\text{var.} & \{\mathbf{p}_i^{(v)}\}, \ \{\mathbf{q}_{i'}^{(v)}\}, \mathbf{z}_i.
\end{array}
\end{equation}

\subsection*{Update the Optimization Parameters}
The Lagrangian function for the optimization problem Eq. $\eqref{distri_opt4}$ is given by
\begin{equation}
\label{eq:l_p}
    \mathcal{L}_p(\{\mathbf{p}_i\}, \{\mathbf{q}_{i'}\}, \{\lambda_i^{l}\}, \{\mathbf{z}_i\}) = F(\{\mathbf{p}_i\},\{\mathbf{z}_i\},\{\mathbf{q}_{i'}\}) 
    + \sum_{i=1}^{M} \sum_{l=1}^{|\mathcal{P}_i^t|} \lambda_i^l \left( \sum_{v=1}^{N_w} \mathbf{a}_i^{l(v)\top}\mathbf{p}_i^{(v)} + \mathbf{b}_i^{l\top} \mathbf{z}_i + c_i^l \right).
\end{equation}

Following \cite{xu2023unified}, the regularized version of (\ref{eq:l_p}) is as follows,
\begin{equation}
\label{eq:l_p2}
\tilde{\mathcal{L}}_p(\{\mathbf{p}_i\}, \{\mathbf{q}_{i'}\}, \{\lambda_i^{l}\}, \{\mathbf{z}_i\})
= \mathcal{L}_p - \sum_{i=1}^{N_w} \sum_{l=1}^{|\mathcal{P}_i^t|} \frac{c_1}{2} \| \lambda_i^{l} \|^2,
\end{equation}
where $\lambda_i^l \in \mathbb{R}^1$ is the Lagrangian dual variable and satisfy $\lambda_i^l > 0$ due to the inequation constraint. $c_1$ denotes the regularization terms. Note that lines 3 and 4 of the constraint in $\eqref{distri_opt4}$ are omitted here, as the constraint can be replaced by adding a projection to the optimization variable.

Considering the heterogeneity of the workers in terms of computational capacity and network, the proposed algorithm allows asynchronous updating of the variables, thus avoiding the stragglers problem. 

\textit{For all workers}: In the $k+1$th iteration, the worker in $\mathcal{A}^{k+1}$ updates their local parameters as,
\begin{align}
\label{update_p_local}
    \mathbf{p}_i^{(v)k+1} & = \operatorname{proj}_{\mathcal{P}} \left(\mathbf{p}_i^{(v)k} 
     - \eta_{\mathbf{p}_i^{(v)}} \left( \nabla_{\mathbf{p}_i^{(v)}}  \tilde{\mathcal{L}}_p( \{\mathbf{p}_i^{\hat{k}_v}\}, \{\mathbf{q}_{i'}^{\hat{k}_v}\}, \{\lambda_i^{l \hat{k}_v}\}, \{\mathbf{z}_i^{\hat{k}_v}\}) + \psi \text{sign}(\mathbf{z}_i^{\hat{k}_v}-\mathbf{p}_i^{\hat{k}_v}) \right) \right), \\
\label{update_q_local}
    \mathbf{q}_{i'}^{(v) k+1} &= \operatorname{proj}_{\mathcal{P}}\left(\mathbf{q}_{i'}^{(v) k}
    - \eta_{\mathbf{q}_{i'}^{(v)}} \nabla_{\mathbf{q}_{i'}^{(v)}}  \tilde{\mathcal{L}}_p(\{\mathbf{p}_i^{\hat{k}_v}\}, \{\mathbf{q}_{i'}^{\hat{k}_v}\}, \{\lambda_i^{l \hat{k}_v}\}, \{\mathbf{z}_i^{\hat{k}_v}\})\right),
\end{align}
where $\hat{k}_v$ is the last reachable iteration of worker $v$. The rest of the workers keep their local parameters as they were last updated. The worker then sends the updated $\mathbf{p}_i^{(v) k+1}, \mathbf{q}_{i'}^{(v) k+1} $ to the parameter server.

\textit{For parameter server}: After the parameter server receives the updated local variables from all reachable workers, it updates the consensus variables as well as the dual variables, 
\begin{small}
\begin{align}
\label{update_z}
    \mathbf{z}_i^{k+1} &= \operatorname{proj}_{\mathcal{P}}\left(\mathbf{z}_i^{k} 
    - \eta_{\mathbf{z}_i} \nabla_{\mathbf{z}_i} \left( \tilde{\mathcal{L}}_p(\{\mathbf{p}_i^{k+1}\}, \{\mathbf{q}_{i'}^{k+1}\}, \{\lambda_i^{l k}\}, \{\mathbf{z}_i^{k}\}) + \psi \big( \sum_{i \in \{1,...,N_w\}} \text{sign}(\mathbf{z}^{k}-\mathbf{p}_i^{k+1}) + \sum_{j \in \{1,...,B\} } \text{sign}(\mathbf{z}^{k}-\mathbf{p}_j^{k+1})\big) \right) \right), \\
\label{update_lambda}
    \lambda_i^{l k+1} &= \operatorname{proj}_{\lambda}\left(\lambda_i^{l k} 
    + \eta_{\lambda_i^l} \nabla_{\lambda_i^l} \tilde{\mathcal{L}}_p(\{\mathbf{p}_i^{k+1}\}, \{\mathbf{q}_{i'}^{k+1}\}, \{\lambda_i^{l k}\}, \{\mathbf{z}_i^{k+1}\})\right),
\end{align}
\end{small}
where $proj_{\lambda}$ is a projection that guarantee $\lambda_i^l > 0$. The gradient involved in the update procedure can be calculated as following,
\begin{align}
    \nabla_{\mathbf{p}_i^{(v)}}  \tilde{\mathcal{L}}_p(\{\mathbf{p}_i\}, \{\mathbf{q}_{i'}\}, \{\lambda_i^{l}\}, \{\mathbf{z}_i\}) 
    &= \nabla_{\mathbf{p}_i^{(v)}} f(\mathbf{p}_1^{(v)} , \cdots, \mathbf{p}_M^{(v)}, \mathbf{q}_1^{(v)}, \cdots, \mathbf{q}_U^{(v)}) + \sum_{l=1}^{L} \lambda_i^l \mathbf{a}_i^{l (v)}, \\
    \nabla_{\mathbf{q}_{i'}^{(v)}}  \tilde{\mathcal{L}}_p(\{\mathbf{p}_i\}, \{\mathbf{q}_{i'}\}, \{\lambda_i^{l}\}, \{\mathbf{z}_i\}) 
    &= \nabla_{\mathbf{q}_{i'}^{(v)}} f(\mathbf{p}_1^{(v)} , \cdots, \mathbf{p}_M^{(v)}, \mathbf{q}_1^{(v)}, \cdots, \mathbf{q}_U^{(v)}), \\
    \nabla_{\mathbf{z}_i} \tilde{\mathcal{L}}_p(\{\mathbf{p}_i\}, \{\mathbf{q}_{i'}\}, \{\lambda_i^{l}\}, \{\mathbf{z}_i\}) 
    &= \sum_{l=1}^{L} \lambda_i^l \mathbf{b}_i^l, \\
    \nabla_{\lambda_i^l} \mathcal{L}_p(\{\mathbf{p}_i\}, \{\mathbf{q}_{i'}\}, \{\lambda_i^{l}\}, \{\mathbf{z}_i\}) 
    &= \sum_{v=1}^{N_w} \mathbf{a}_i^{l (v)\top}\mathbf{p}_i^{(v)} + \mathbf{b}_i^{l\top} \mathbf{z}_i + c_i^l - c_1^t.
\end{align}
For worker $v$, the estimated gradient $\nabla_{\mathbf{p}_i^{(v)}}f(\cdot)$ are given by $\eqref{black_box}$. For brevity, the following equation omits the superscript $(v)$ of all variables. 
\begin{equation}
    \label{black_box}
    \begin{array}{ll}
    \hspace{-10pt} & \nabla_{\mathbf{p}_i} f(\mathbf{p}_{1:M}, \mathbf{q}_{1:U}) \\
    \hspace{-10pt} = & \nabla_{\mathbf{p}_i} \mathbb{E}_{J_{1:M}, K_{1:U}}[ \mathcal{L}_{ce}(J_{1:M},K_{1:U})] \\
    \hspace{-10pt} = & \int \mathcal{L}_{ce}(j_{1:M}, k_{1:U}) P(k_{1:U}) \nabla_{\mathbf{p}_i}P(j_{1:M}) dj_{1:M} dk_{1:U} \\ 
    \hspace{-10pt} = & \int \mathcal{L}_{ce}(j_{1:M}, k_{1:U}) P(k_{1:U}) \frac{P(j_{1:M})}{P(j_{1:M})} \nabla_{\mathbf{p}_i} P(j_{1:M}) dj_{1:M} dk_{1:U} \\
    \hspace{-10pt} = & \int P(k_{1:U}) \mathcal{L}_{ce}(j_{1:M}, k_{1:U}) P(j_{1:M}) \nabla_{\mathbf{p}_i}\log P(j_{1:M})  dj_{1:M} dk_{1:U} \\
    \hspace{-10pt} = & \int P(k_{1:U}) \mathbb{E}_{J_{1:M}} \left[ \mathcal{L}_{ce}(j_{1:M}, k_{1:U}) \nabla_{\mathbf{p}_i}\log P(j_{1:M}) \right]  dk_{1:U} \\
    \hspace{-10pt} = & \mathbb{E}_{K_{1:U}} \left[ \mathbb{E}_{J_{1:M}} \left[  \mathcal{L}_{ce}(j_{1:M}, k_{1:U}) \nabla_{\mathbf{p}_i}\log P(j_{1:M}) \right] \right]\\
    \hspace{-10pt} = & \mathbb{E}_{K_{1:U}} \left[ \mathbb{E}_{J_{1:M}} \left[ \mathcal{L}_{ce}(j_{1:M}, k_{1:U}) \nabla_{\mathbf{p}_i}\log \prod_{i=1}^{M} P_{\mathbf{p}_i}(j_i) \right] \right]\\
    \hspace{-10pt} = & \mathbb{E}_{K_{1:U}, J_{1:M}} \left[ \mathcal{L}_{ce}(j_{1:M}, k_{1:U}) \nabla_{\mathbf{p}_i}\log P_{\mathbf{p}_i}(j_i) \right], \\
\end{array}
\end{equation}
where we tentatively and informally use variables with subscripts $1:n$ to denote the $n$ variables that have subscripts $1,\cdots,n$, respectively. We also abbreviate $dj_1 \cdots dj_M$, $dk_1 \cdots dk_U$ to $dj_{1:M}$ and $dk_{1:U}$, respectively.

\subsection*{Update the Polyhedron}
As local parameters and consensus variables are continuously updated across workers, additional cutting planes must be incorporated to form a compact feasible region for Eq. $\eqref{h_function}$. To mitigate computational burden, cutting planes with small weights will be removed.

\textit{Add new cutting plane}: 
For a newly updated $(\{\mathbf{p}_i^{(v) t+1}\}, \mathbf{z}_i^{t+1})$, we first evaluate whether this solution is feasible for the optimization problem Eq. $\eqref{distri_opt4}$. If not, i.e., $h_i(k_1, \cdots, k_U, \{\mathbf{p}_i^{(v) t+1}\}, \mathbf{z}_i^{t+1}) > \epsilon$. Then we need to generate a new cutting plane for separating the points $(\{\mathbf{p}_i^{(v) t+1}\}, \mathbf{z}_i^{t+1})$ outside the $\mathcal{P}_i$. The newly generated cutting plane with index $l$ needs to satisfy,
\begin{equation}
\label{cuttingplane}
\begin{cases}
 \sum_{v=1}^{N_w} \mathbf{a}_i^{l(v)\top}\mathbf{p}_i^{(v)} + \mathbf{b}_i^{l\top} \mathbf{z}_i + c_i^l \leq 0; \ \forall (\{\mathbf{p}_i^{(v)}\}, \mathbf{z}_i) \in \mathcal{P}_i \\
 \sum_{v=1}^{N_w} \mathbf{a}_i^{l (v)\top} \mathbf{p}_i^{(v) t+1} + \mathbf{b}_i^{l\top} \mathbf{z}_i^{t+1} + c_i^l > 0.
\end{cases}
\end{equation}
Since the $h_i(k_1, \cdots, k_U, \{\mathbf{p}_i^{(v)}\}, \mathbf{z}_i)$ is convex with respect to $\{\mathbf{p}_i^{(v)}\}, \mathbf{z}_i$, then we have,
\begin{equation}
\label{cuttingplane2}
\begin{array}{ll}
    \hspace{-10pt} h_i(k_1, \cdots, k_U, \{\mathbf{p}_i^{(v)}\}, \mathbf{z}_i) \geq h_i(k_1, \cdots, k_U, \{\mathbf{p}_i^{(v) t+1}\}, \mathbf{z}_i^{t+1})  + \begin{bmatrix} \{ \frac{\partial h_i(k_1, \cdots, k_U, \{\mathbf{p}_i^{(v) t+1}\}, \mathbf{z}_i^{t+1})}{\partial \mathbf{p}_i^{(v)}} \}  \\ \frac{\partial h_i(k_1, \cdots, k_U, \{\mathbf{p}_i^{(v) t+1}\}, \mathbf{z}_i^{t+1})}{\partial\mathbf{z}_i} \end{bmatrix}^{\top} \hspace{-5pt} \left( \begin{bmatrix} \{\mathbf{p}_i^{(v)}\} \\ \mathbf{z}_i \end{bmatrix} - \begin{bmatrix} \{\mathbf{p}_i^{(v) t+1}\} \\ \mathbf{z}_i^{t+1} \end{bmatrix} \right).
\end{array}
\end{equation}
In summary, the newly added cutting plane satisfies,
\begin{equation}
\begin{array}{ll}
\hspace{-10pt} h_i(k_1, \cdots, k_U, \{\mathbf{p}_i^{(v) t+1}\}, \mathbf{z}_i^{t+1}) + \begin{bmatrix} \{ \frac{\partial h_i(k_1, \cdots, k_U, \{\mathbf{p}_i^{(v) t+1}\}, \mathbf{z}_i^{t+1})}{\partial \mathbf{p}_i^{(v)}} \}  \\ \frac{\partial h_i(k_1, \cdots, k_U, \{\mathbf{p}_i^{(v) t+1}\}, \mathbf{z}_i^{t+1})}{\partial\mathbf{z}_i} \end{bmatrix}^{\top}  \hspace{-5pt} 
 \left( \begin{bmatrix} \{\mathbf{p}_i^{(v)}\} \\ \mathbf{z}_i \end{bmatrix}   - \begin{bmatrix} \{\mathbf{p}_i^{(v) t+1}\} \\ \mathbf{z}_i^{t+1} \end{bmatrix} \right) 
\leq \epsilon.
\end{array}
\end{equation}
Combine Eq. $\eqref{cuttingplane}$ with Eq. $\eqref{cuttingplane2}$, the parameters in the equation $\sum_{v=1}^{N_w} \mathbf{a}_i^{l
 (v)\top}\mathbf{p}_i^{(v)} + \mathbf{b}_i^{l\top} \mathbf{z}_i + c_i^l \leq 0$ are calculated as follows,
\begin{align}
\label{cutting_a}
    &\mathbf{a}_i^{l (v)} = \frac{\partial h_i(k_1, \cdots, k_U, \{\mathbf{p}_i^{(v) t+1}\}, \mathbf{z}_i^{t+1})}{\partial \mathbf{p}_i^{(v)}}, \\
\label{cutting_b}
    &\mathbf{b}_i^l = \frac{\partial h_i(k_1, \cdots, k_U, \{\mathbf{p}_i^{(v) t+1}\}, \mathbf{z}_i^{t+1})}{\partial\mathbf{z}_i}, \\
\label{cutting_c}
    &c_i^l = h_i(k_1, \cdots, k_U, \{\mathbf{p}_i^{(v) t+1}\}, \mathbf{z}_i^{t+1}) 
     - \begin{bmatrix} \left\{ \frac{\partial h_i(k_1, \cdots, k_U, \{\mathbf{p}_i^{(v) t+1}\}, \mathbf{z}_i^{t+1})}{\partial \mathbf{p}_i^{(v)}} \right\}  \\ \frac{\partial h_i(k_1, \cdots, k_U, \{\mathbf{p}_i^{(v) t+1}\}, \mathbf{z}_i^{t+1})}{\partial\mathbf{z}_i} \end{bmatrix}^{\top} \begin{bmatrix} \{\mathbf{p}_i^{(v) t+1}\} \\ \mathbf{z}_i^{t+1} \end{bmatrix}.
\end{align}

\textit{Delete invalid cutting planes}:
When the dual variable $\lambda_l$ of the $l$th cutting plane is less than a given threshold before and after the $t+1$th update, then the cutting plane $l$ will no longer be used for subsequent calculations.

\subsection*{The Proposed AsynDBT Approach}
To address the straggler problem in federated learning, we propose \textit{AsynDBT}, an asynchronous algorithm that enables the parameter server to communicate with only a subset of available workers per update, bypassing the need to wait for results from all workers.

We define a set of hyperparameters for the AsynDBT. The polyhedron is updated every $\delta$ step, $\tau$ is the maximum update interval for each worker, and $\gamma$ denotes the threshold for determining the invalidity of the cutting plane. Note that we denote the last communication time of worker $v$ by $\hat{k}^{(v)}$. Pseudo-code for asynBDT is shown in Algorithm \ref{alg:1}.

\begin{algorithm} 
\caption{The proposed AsynDBT algorithm}
\label{alg:1}
\begin{algorithmic}[1]
\STATE \textbf{Initialization:} local variables $\{\mathbf{p}_i^{(v)0}\}$ and $\{\mathbf{q}_{i'}^{(v)0}\}$ of workers, Consensus variables $\{\mathbf{z}_i^0\}$, dual variables $\{\lambda_i^{l0}\}$ and $\{\boldsymbol{\rho}_{i}^{(v)0}\}$ in the parameter server.
\REPEAT
\STATE Check network and cloud service accessibility of all workers to form a set of reachable nodes $\mathcal{A}^k$;
\STATE $\mathcal{A}^k \leftarrow \text{Add}(\mathcal{A}^k, v), \forall v: \hat{k}^{(v)} \geq \tau$;
\STATE Each worker in $\mathcal{A}^k$ updates its local variables according to Eq. $\eqref{update_p_local}$ and Eq. $\eqref{update_q_local}$ and transmits its updated variables $\{\mathbf{p}_i^{(v)k}\}$ and $\{\mathbf{q}_{i'}^{(v)k}\}$ to parameter server;
\STATE Parameter server updates $\{\mathbf{z}_i^k\}$ and $\{\lambda_i^{lk}\}$ by Eq. $\eqref{update_z}$ and Eq. $\eqref{update_lambda}$;
\STATE Parameter server broadcasts variables $\{\mathbf{z}_i^k\}$ to workers in $\mathcal{A}^k$;
\IF{$k+1$ mod $\delta$ == $0$}
\STATE Parameter server estimates $\phi_i(\cdot)$ by K step gradient descent according to Eq. $\eqref{inner_p}$, Eq. $\eqref{inner_z}$, and Eq. $\eqref{inner_rho}$;
\STATE Parameter server removes invalid cutting plane $l$ when $\sum_{i=1}^{M} \lambda_i^{lt} < \gamma$;
\STATE Parameter server updates new cutting plane by Eq. $\eqref{cutting_a}$, Eq. $\eqref{cutting_b}$, and Eq. $\eqref{cutting_c}$, then broadcasts $\{\mathbf{a}_i^{l (v)}\}$, $\{\mathbf{b}_i^l\}$ and $c_i^l$ to all workers;
\ENDIF
\STATE $k = k + 1$;
\UNTIL{termination.}   
\end{algorithmic} 
\end{algorithm}

\subsection*{Proof of Convergence}

\begin{theorem}[Convergence]
As the cutting plane continuously adds to the polyhedron, the optimal objective value in the approximation problem Eq. $\eqref{distri_opt4}$ converges monotonically.
\end{theorem}

\begin{proof}
Let $\mathcal{R}_i$ denote the feasible region of Eq. $\eqref{distri_opt3}$ and $\mathcal{P}_i$ represent the polyhedral approximation of Eq. $\eqref{distri_opt4}$'s feasible region. The polyhedron $\mathcal{P}^{n\delta}_i$ obtained at iteration $n\delta$ contains the point $({\mathbf{p}_i^{(v) n\delta}}, \mathbf{z}_i^{n\delta})$, implying the nested containment relationship: $\mathcal{P}^{0}_i \supseteq \mathcal{P}^{\delta}_i \supseteq \cdots \supseteq \mathcal{P}^{n\delta}_i \supseteq \mathcal{R}_i$. For iteration $n\delta$, denote the optimal objective value of Eq. $\eqref{distri_opt4}$ as $F({\mathbf{p}_i^{n\delta\star}}, {\mathbf{z}i^{n\delta \star}}, {\mathbf{q}{i'}^{n\delta \star}})$, then the following inequality holds,
\begin{equation}
    \hspace{-10pt} F(\{\mathbf{p}_i^{0\star}\}, \{\mathbf{z}_i^{0\star}\}\{\mathbf{q}_{i'}^{0 \star}\}) \leq F(\{\mathbf{p}_i^{\delta\star}\}, \{\mathbf{z}_i^{\delta \star}\}\{\mathbf{q}_{i'}^{\delta \star}\}) 
    \hspace{-10pt} \leq \cdots \leq  F(\{\mathbf{p}_i^{n\delta\star}\}, \{\mathbf{z}_i^{n\delta \star}\}\{\mathbf{q}_{i'}^{n\delta \star}\}).
\end{equation}
Let the optimal objective value of Eq. $\eqref{distri_opt3}$ be $F^{\star}$, then we have 
\begin{equation}
    \hspace{-10pt} F^{\star} / F(\{\mathbf{p}_i^{0\star}\}, \{\mathbf{z}_i^{0\star}\}\{\mathbf{q}_{i'}^{0 \star}\}) \geq  F^{\star} / F(\{\mathbf{p}_i^{\delta\star}\}, \{\mathbf{z}_i^{\delta \star}\}\{\mathbf{q}_{i'}^{\delta \star}\}) 
     \geq  \cdots \geq  F^{\star} /  F(\{\mathbf{p}_i^{n\delta\star}\}, \{\mathbf{z}_i^{n\delta \star}\}\{\mathbf{q}_{i'}^{n\delta \star}\}) \geq  \beta.
\end{equation}
From the above non-increasing sequence, it can be seen that when $n \rightarrow \infty$, the optimal objective value of Eq. $\eqref{distri_opt4}$ converges to $\beta, \beta \geq 1$.
\end{proof}

\begin{assumption}
\label{ass1}
    \textbf{(Lipschitz continuous)} We assume that $\mathcal{L}_p$ has $L'$-Lipschitz continuous gradients ($L>0$). For any $\mathbf{x}$ and $\mathbf{x}'$, it satisfies
    \begin{equation}
        \| \nabla \mathcal{L}_p(\mathbf{x}) - \nabla \mathcal{L}_p(\mathbf{x}') \| \leq L' \| \mathbf{x} - \mathbf{x}' \|
    \end{equation}
\end{assumption}

\begin{assumption}
    \label{ass2}
    \textbf{(Boundedness)} The optimization variables are bounded, i.e. $\|\mathbf{p}_i^{(v)}\|^2 \leq \alpha_1$, $\|\mathbf{z}_{i}\|^2 \leq \alpha_1$, $\|\mathbf{q}_{i'}^{(v)}\|^2 \leq \alpha_2$, $\|\lambda_i^l\|^2 \leq \alpha_3$), and before obtaining the $\epsilon$-stationary point the variables in parameter server satisfy that $\| \mathbf{z}_i^{t+1} - \mathbf{z}_i^t \|^2 + \sum_{l} \| \lambda_i^{l t+1} - \lambda_i^{l t} \|^2 \geq \xi$, where $\xi > 0$ is a relative small constant. The change of the variables in the parameter server is upper bounded within $\tau$ iterations, i.e., $\| \mathbf{z}_i^t - \mathbf{z}_i^{t-k} \|^2 \leq \tau \xi k_1$, $ \sum_{l} \| \lambda_i^{lt} - \lambda_i^{l t-k} \|^2\leq \tau \xi k_1$, where $1 < k < \tau$ and $k_1 > 0$ is a constant. 
\end{assumption}

\begin{theorem}[Iteration Complexity]
\label{theorem2}
Let Assumptions \ref{ass1} and \ref{ass2} hold. Then, the iteration complexity of our proposed AsynDBT to obtain $\epsilon$-stationary point is bounded by,
\begin{align}
\label{eq:A97}
T( \epsilon ) \! \sim  \! \mathcal{O}\Bigg(\max  \Big\{ {(\frac{{4M\alpha_3}}{{{{\eta _{\lambda}}}^2}} \!+\! \frac{{4N\alpha_4}}{{{{\eta _{\boldsymbol{\theta }}}}^2}})^2}\frac{1}{{{\epsilon ^2}}}, \nonumber {(\frac{{4{{{(d_7 + \frac{{{{\eta _{\boldsymbol{\theta }}}}(N  -  S){{L}^2}}}{2})}}} (\mathop d\limits^ -  +  k_d\tau(\tau  -  1))  {d_6}}}{{{\epsilon}}} + (T_1 + 2)^{\frac{1}{2}})^2} \Big\} \Bigg),  
\end{align}
such that $||\nabla G^t||^2  \le \epsilon $.
\end{theorem}

We provide a detailed derivation process in the Appendix\footnote{https://github.com/maggiemh/AsynDBT}, which consists of the following four steps. First, we derive Lemma 1 (see Appendix C) based on Assumption \ref{ass1} and Assumption \ref{ass2}. Next, by combining the Cauchy-Schwarz inequality and Lemma 1, we obtain Lemma 2 (see Appendix D). Furthermore, leveraging Lemma 1 along with Lemma 2, we derive Lemma 3 (see Appendix E). Finally, by integrating the above three lemmas, we formally derive Theorem \ref{theorem2} (see Appendix F) for our proposed AsynDBT. 

According to Theorem \ref{theorem2}, the upper bound on the iteration complexity of AsynDBT for achieving an $\epsilon$-stationary point is $O(1/\epsilon^2)$. The iteration complexity is affected by several parameters in AsynDBT, such as $\epsilon$, $S$,  $N$. $S$ denotes the number of active workers in each iteration and $N$ represents the number of workers in federated learning framework. When a smaller $\epsilon$ is chosen, the iteration complexity increases. On the other hand, increasing the number of active workers $S$ in each iteration reduces the iteration complexity. However, as the number of clients $N$ grows, the iteration complexity rises exponentially.

\section*{Experiments}
To assess the performance of AsynDBT in domain-specific and Natural Language Understanding (NLU) tasks, we conducted experiments on six classification datasets.

\subsection*{Datasets Description}
The first dataset focuses on the task of terminological relationship recognition in the 5G network. In our preliminary work, we created a 5G terminology knowledge graph based on 3GPP protocol texts. This graph labels terminology pairs as either \texttt{relevant} or \texttt{irrelevant}, and provides a corresponding segment of 3GPP protocol text for each pair to assist with reasoning.
For NLU tasks, we utilize five datasets from the GLUE benchmark \cite{wang2018glue}.
These datasets cover a range of tasks: COLA and SST-2 for single-sentence classification, MRPC and QQP for syntactic comparison, and QNLI for inference.
Table \ref{tab:templates} provides an overview of the prompt templates used for all datasets in our experiments. In these templates, the placeholder \textcolor{mygreen}{[VAR]} indicates the prompt fragment that is subject to optimization.

\begin{table}[h]
\centering
\caption{The prompt templates for all the datasets}
\begin{tabular}{m{0.09\linewidth}m{0.8\linewidth}}
\toprule
Datasets & Templates \\ \midrule
5G   & In 5G network, Whether \textcolor{myblue}{[WORD1]} related   to \textcolor{myblue}{[WORD2]}? Some contextual information: \textcolor{myblue}{[TEXT]}. Respond ONLY with "Yes" or "No". Note:   \textcolor{mygreen}{[VAR]}.     \\
COLA & Is this sentence \textcolor{myblue}{[SENTENCE1]}   grammatically correct? Respond ONLY with "Yes" or "No".   Note: \textcolor{mygreen}{[VAR]}.                                           \\
SST2 & How is the sentiment of the sentence \textcolor{myblue}{[SENTENCE1]}? Respond ONLY with "Great" or   "Terrible".  Note: \textcolor{mygreen}{[VAR]}.                                     \\
MRPC & Whether sentence \textcolor{myblue}{[SENTENCE1]}   and sentence \textcolor{myblue}{[SENTENCE2]} are semantically the same? Respond ONLY with "Yes" or   "No".  Note: \textcolor{mygreen}{[VAR]}.           \\
QQP  & Whether sentence \textcolor{myblue}{[SENTENCE1]}   and sentence \textcolor{myblue}{[SENTENCE2]} are paraphrased from each other? Respond ONLY with "Yes" or   "No".  Note: \textcolor{mygreen}{[VAR]}.    \\
QNLI & Whether sentence \textcolor{myblue}{[SENTENCE1]} and sentence \textcolor{myblue}{[SENTENCE2]} have semantic entailment relations? Respond ONLY with "Yes" or   "No".  Note: \textcolor{mygreen}{[VAR]}. \\ \bottomrule
\end{tabular}
\label{tab:templates}
\end{table}

For the 5G dataset, which has distinct training, validation, and test sets, 10 samples per class were randomly selected from the training and validation sets, and 200 samples were randomly chosen from the original test set. For each GLUE benchmark dataset, we randomly selected 10 samples per class from the original training set to create a reduced training set. We randomly chose 10 samples per class from the original labeled validation set for validation and 200 samples randomly drawn from the remaining validation set for test. Additionally, we randomly selected 50 samples from the remaining original training set for each class to form the ICL training set. Table \ref{tab:sample_ratio} shows the details for all the test datasets.

\begin{table}[h]
\centering
\caption{The details of all the test datasets}
\begin{tabular}{@{}lll@{}}
\toprule
Datasets & \multicolumn{2}{l}{Number of a samples} \\ \midrule
5G       & Yes: 99          & No: 101              \\
COLA     & Yes: 141         & No: 59                \\
SST2     & Great: 99        & Terrible: 101        \\
MRPC     & Yes: 135         & No: 65               \\
QQP      & Yes: 89          & No: 111              \\
QNLI     & Yes: 108         & No: 92               \\ \bottomrule
\end{tabular}
\label{tab:sample_ratio}
\end{table}

\subsection*{Baselines}
\begin{itemize}
    \item \textbf{RoBERTa} \cite{roberta}: We convert all sentences in a test sample into vectors using pre-trained RoBERTa\footnote{\url{https://huggingface.co/roberta-base}}. These vectors are summed and fed into a single-layer MLP classifier, which is trained using cross-entropy loss.
    
    \item \textbf{ManualPrompt (MP)}: The prompt template is as described in Table \ref{tab:templates}, but does not include the \texttt{Note} and its aftermath.
    
    \item \textbf{Zero-shot CoT} \cite{kojima2022large}: Based on the ManualPrompt method, we append the text \texttt{Let's think step by step, first output your analysis, and then output the final answer} to each test sample.

    \item \textbf{Random ICL}: This method involves randomly selecting samples from the ICL training set as in-context samples, ensuring only one sample per class. Note that the demonstrations are randomly selected in each run.

    \item \textbf{KATE} \cite{liu2022makes}: This approach first maps the test sample and ICL training set into vectors using a pre-trained RoBERTa. The 5 nearest neighbors of the test sample in the vector space are selected as in-context samples.

    \item \textbf{BDPL} \cite{diao2022black}: In this method, a prompt fragment is learned using the policy gradient approach and then appended to the original prompt.
    
    \item \textbf{AdaICL} \cite{mahankali2024one}: It is a model-adaptive method that uses LLMs to predict unlabeled data, assigning an uncertainty score to each instance.
\end{itemize}

Besides the proposed AsynDBT approach, its centralized version, cenDBT, also participates in the comparison. cenDBT optimizes the parameters only on the server, so there is no need to optimize the consensus variable $\{\mathbf{z}_i\}$ and the corresponding dual variable $\{\boldsymbol{\rho}_i^{(v)}\}$.

\subsection*{Experimental Details}

\textit{Hyperparameter Setting}: The API of LLM service is supported by qwen-max\footnote{\url{https://tongyi.aliyun.com}}. Both BDPL and the proposed method are trained using the Adam optimizer with a learning rate of $10^{-4}$. Specifically, the Lagrangian dual variable of the proposed method has a learning rate of $10^{-1}$. The learning rate of RoBERTa is $2 \times 10^{-5}$. The maximum of epochs for all the algorithms is 500 and the early stopping technique is used on the validation set. In parameter settings in $\eqref{bilevel1}$ are $M=75$, $N=100$ and $V=50$. The value of $U$ is determined by the dataset. the setting of BDPL is consistent with the proposed method. Besides, we use accuracy to evaluate the classification performance of different models.

\textit{Experimental Environment}: We conduct all experiments
on a Linux server with four 12 GB GPUs with NVIDIA
TITAN X (Pascal). Besides, we use the deep learning
framework of Pytorch 1.6.0 with the programming language of python 3.7.
\begin{figure}[t]
\centering
\includegraphics[width=0.4\linewidth]{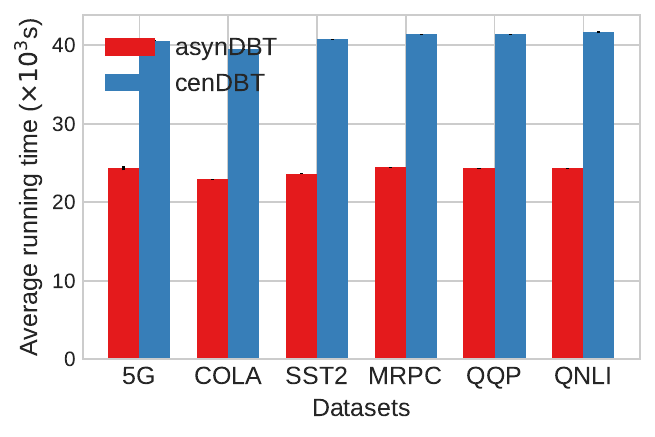}
\caption{Comparison of time consumption.}
\label{fig:compare_time}
\end{figure}
\begin{figure}[t]
    \centering    
    \begin{subfigure}[b]{0.4\linewidth}
        \includegraphics[width=\linewidth]{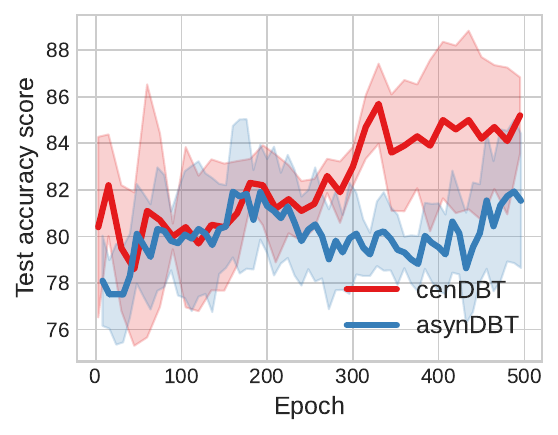}
        \caption{Accuracy score}
    \end{subfigure}
    \begin{subfigure}[b]{0.4\linewidth}
        \includegraphics[width=\linewidth]{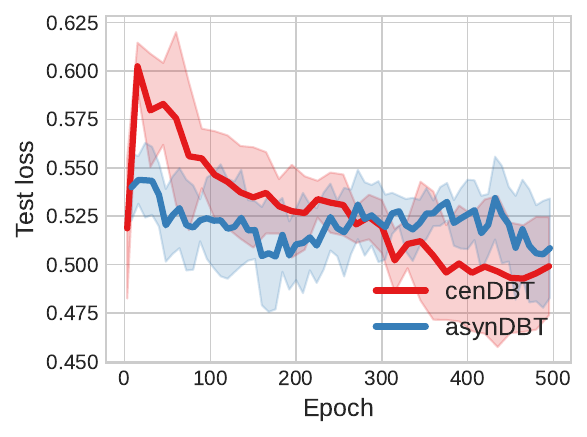}
        \caption{Cross-entropy loss}
    \end{subfigure}
    \caption{(a) Accuracy and (b) Cross-entropy loss on the 5G dataset.}
    \label{fig:5g}
\end{figure}

\begin{figure}[t]
    \centering
    \begin{subfigure}[b]{0.4\linewidth}
        \includegraphics[width=\linewidth]{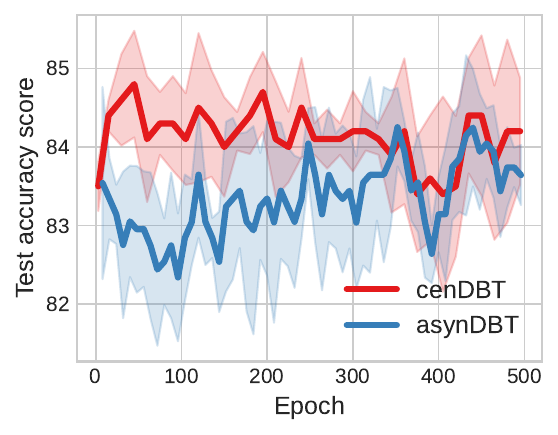}
        \caption{Accuracy score}
    \end{subfigure}
    \begin{subfigure}[b]{0.4\linewidth}
        \includegraphics[width=\linewidth]{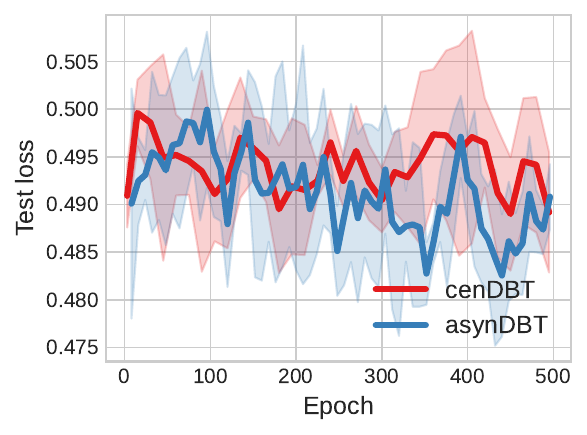}
        \caption{Cross-entropy loss}
    \end{subfigure}
    \caption{(a) Accuracy and (b) Cross-entropy loss on the COLA dataset.}
    \label{fig:cola}
\end{figure}
\begin{figure}[t]
    \centering
    \begin{subfigure}[b]{0.4\linewidth}
        \includegraphics[width=\linewidth]{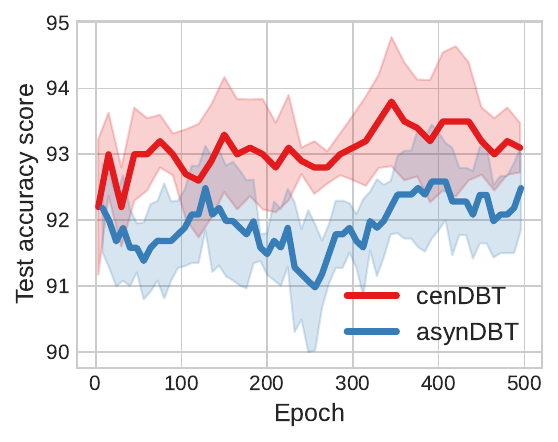}
        \caption{Accuracy score}
    \end{subfigure}
    \begin{subfigure}[b]{0.4\linewidth}
        \includegraphics[width=\linewidth]{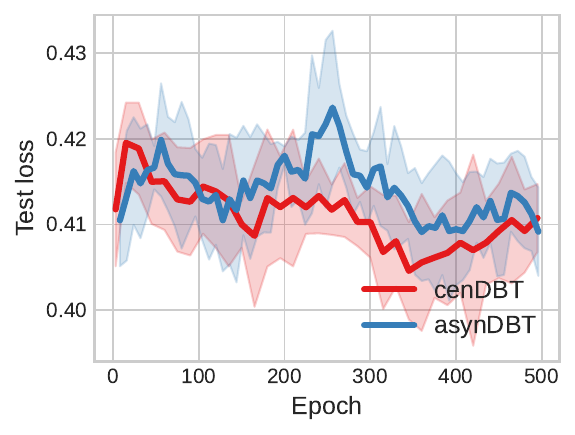}
        \caption{Cross-entropy loss}
    \end{subfigure}
    \caption{(a) Accuracy and (b) Cross-entropy loss on the SST2 dataset.}
    \label{fig:sst2}
\end{figure}

\section*{Experimental Results}
\subsection*{The Performance of Classification}

\begin{table*}[t]
\centering
\caption{The overall results of different models. }
\begin{tabular}{@{}lllllllc@{}}
\toprule
                       & 5G                                & COLA                              & SST2                              & MRPC                              & QQP                               & QNLI & Average Rank                             \\ \midrule
Roberta \cite{roberta}               & 52.80 \begin{tiny}1.68\end{tiny} & 68.70 \begin{tiny}2.49\end{tiny} & 85.10 \begin{tiny}4.63\end{tiny}  & 65.00 \begin{tiny}3.46\end{tiny} & 55.90 \begin{tiny}1.64\end{tiny} & 51.60 \begin{tiny}3.31\end{tiny} & 8.67\\
Manual Prompt           & 58.00  & 83.00  & 93.50  & 77.00  & 81.00  & 84.00 & 6.17 \\
Zero-shot CoT \cite{kojima2022large}          & 78.50  & 84.50  & 90.50  & 76.00  & 73.50  & 79.00  & 6.50\\
RandomICL              & 69.70 \begin{tiny}0.97\end{tiny} & 71.00 \begin{tiny}1.70\end{tiny} & 72.80 \begin{tiny}1.75\end{tiny} & 78.50 \begin{tiny}0.71\end{tiny} & 82.30 \begin{tiny}1.35\end{tiny} & 87.80 \begin{tiny}0.57\end{tiny} & 5.17 \\
KATE  \cite{liu2022makes}                 & 68.10 \begin{tiny}1.19\end{tiny} & 70.90 \begin{tiny}0.55\end{tiny} & 69.50 \begin{tiny}1.54\end{tiny} & 77.90 \begin{tiny}0.42\end{tiny} & 79.00 \begin{tiny}2.03\end{tiny} & 87.40 \begin{tiny}1.43\end{tiny} & 6.83 \\
BDPL   \cite{diao2022black}                 & 67.20 \begin{tiny}1.35\end{tiny} &\textbf{86.70} \begin{tiny}0.27\end{tiny} & \underline{94.30} \begin{tiny}0.27\end{tiny} & 80.40 \begin{tiny}0.42\end{tiny} & \underline{83.20} \begin{tiny}0.27\end{tiny}  & 85.80 \begin{tiny}0.27\end{tiny} & 3.67 \\ 
AdaICL  \cite{mahankali2024one}                  & 78.43 \begin{tiny}1.03\end{tiny} &84.95 \begin{tiny}0.18\end{tiny} & 91.53 \begin{tiny}0.21\end{tiny} & 81.29 \begin{tiny}0.33\end{tiny} & 81.85 \begin{tiny}0.25\end{tiny}  & 87.58 \begin{tiny}0.23\end{tiny} & 3.83\\ \midrule
\textit{cenDBT(ours)}  & \textbf{87.21} \begin{tiny}1.30\end{tiny} & 84.92 \begin{tiny}0.65\end{tiny} & \textbf{94.40} \begin{tiny}0.74\end{tiny} & \textbf{83.51} \begin{tiny}0.59\end{tiny} & \textbf{83.31} \begin{tiny}0.21\end{tiny}  & \textbf{88.52 }\begin{tiny}0.77\end{tiny} & 1.50 \\
\textit{AsynDBT(ours)} & \underline{85.05} \begin{tiny}1.01\end{tiny} & \underline{85.07} \begin{tiny}0.41\end{tiny} & 92.98 \begin{tiny}0.69\end{tiny} & \underline{82.93} \begin{tiny}1.39\end{tiny} & 81.95 \begin{tiny}0.59\end{tiny}  & \underline{87.79} \begin{tiny}0.43\end{tiny} & 2.67 \\ \bottomrule               
\end{tabular}
\label{tab:acc}
\end{table*}

\begin{figure}[t]
    \centering
    \begin{subfigure}[b]{0.4\linewidth}
        \includegraphics[width=\linewidth]{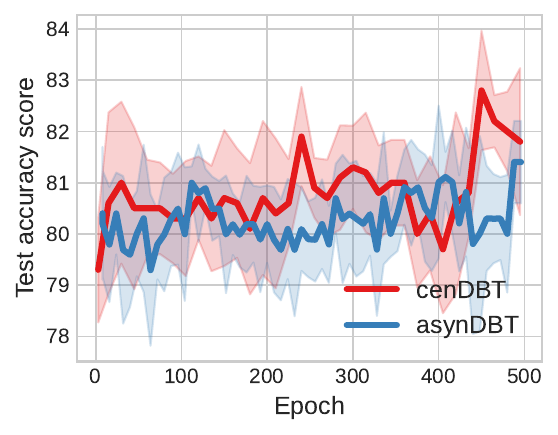}
        \caption{Accuracy score}
    \end{subfigure}
    \begin{subfigure}[b]{0.4\linewidth}
        \includegraphics[width=\linewidth]{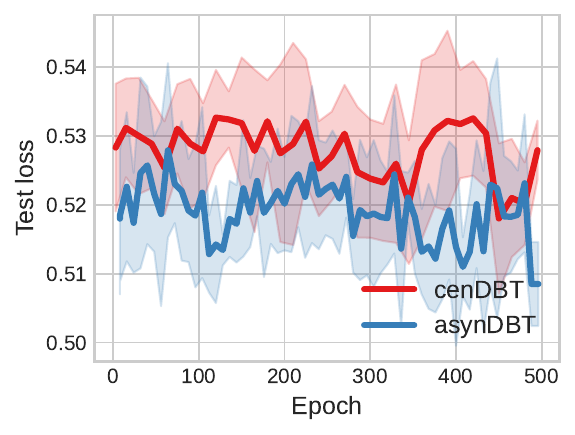}
        \caption{Cross-entropy loss}
    \end{subfigure}
    \caption{(a) Accuracy and (b) Cross-entropy loss on the MRPC dataset.}
    \label{fig:mrpc}
\end{figure}

 \begin{figure}[t]
     \centering
     \begin{subfigure}[b]{0.4\linewidth}
         \includegraphics[width=\linewidth]{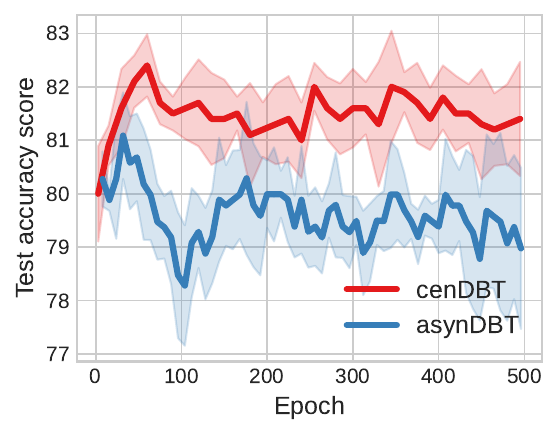}
        \caption{Accuracy score}
     \end{subfigure}
     \begin{subfigure}[b]{0.4\linewidth}
         \includegraphics[width=\linewidth]{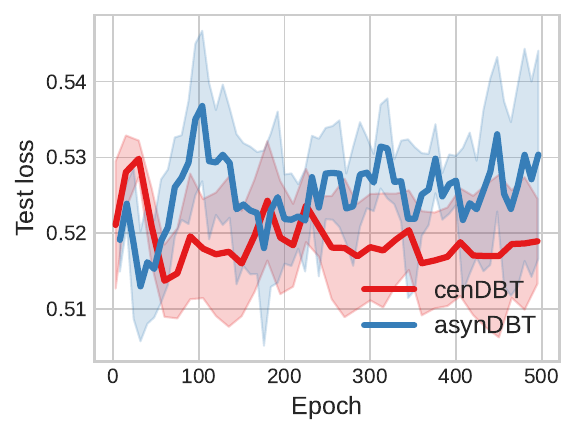}
         \caption{Cross-entropy loss}
     \end{subfigure}
     \caption{(a) Accuracy and (b) Cross-entropy loss on the QQP dataset.}
     \label{fig:qqp}
 \end{figure}

\begin{figure}[t]
    \centering
    \begin{subfigure}[b]{0.4\linewidth}
        \includegraphics[width=\linewidth]{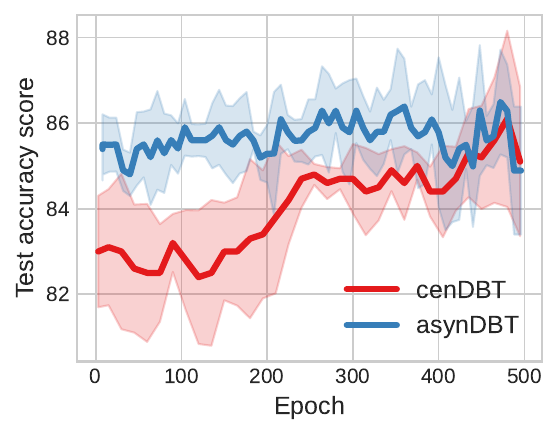}
        \caption{Accuracy score}
    \end{subfigure}
    \begin{subfigure}[b]{0.4\linewidth}
        \includegraphics[width=\linewidth]{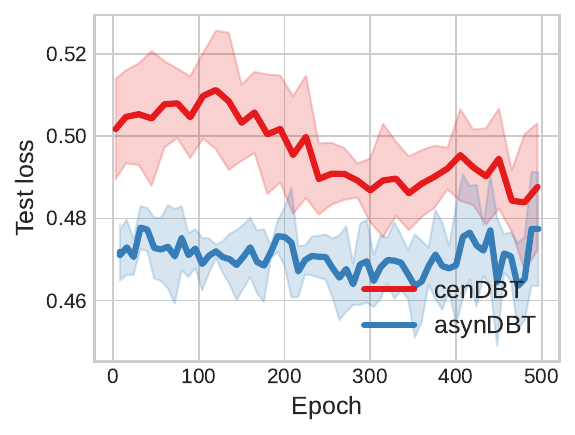}
        \caption{Cross-entropy loss}
    \end{subfigure}
    \caption{(a) Accuracy and (b) Cross-entropy loss on the QNLI dataset.}
    \label{fig:qnli}
\end{figure}

Table \ref{tab:acc} presents the average accuracy and standard deviations of the compared methods over five independent runs, with AsynDBT reporting the average performance of three workers. The optimal and suboptimal performance on each dataset is bolded and underlined, respectively. By ranking the experimental results of each model, we use the average ranking to measure the prediction performance of different models. The output strategy of LLM is set as greedy sampling.

The classification task on the 5G dataset is especifically challenging because it involves two acronyms that are uncommon in general corpora and have not been covered by LLM fine-tuning. The auxiliary text for each pair provides only contextual information, not explicit interpretations. Compared with the optimal baseline, our proposed method improves the accuracy by nearly 10\%. Owing to the isomorphism in data distribution, cenDBT generally outperforms AsynDBT. However, AsynDBT still achieves competitive results relative to the centralized algorithm and finally reaches a suboptimal performance. This highlights both the effectiveness and computational efficiency of the proposed algorithms.

For the GLUE benchmark, our algorithm achieves optimality or suboptimality on almost all dataset. For instance, on MRPC and QNLI datasets, cenDBT attains the highest accuracy, while AsynDBT achieves the next best performance. Among the baselines, BDPL is a prompt-based learning method that guides the model to generate more interpretable outputs by designing specific prompt statements. It optimizes the parameters of a categorical distribution to generate the prompt fragment appended to the original prompt. However, it requires labeled data to optimize the prompt vector, and its initialization can significantly impact experimental results. In contrast, our proposed method adapts to new tasks without parameter updates, saving computational resources and making it suitable for resource-constrained environments, such as edge devices with limited computing power. 

Furthermore, compared to the heuristic KATE shown in Table \ref{tab:acc}, the proposed method consistently refines demonstrations via LLM feedback, substantially improving downstream task accuracy. This finding further highlights the benefit of using diverse in-context samples for downstream tasks and underscores the adaptability of our algorithm in various NLU scenarios.


%

\begin{table}[!t]
\centering
\caption{The results of ablation study.}
\begin{tabular}{@{}lcccccc@{}}
\toprule
      & 5G          & COLA  &   SST2     & MRPC          & QQP           & QNLI  \\ \midrule
cenDBT         & \textbf{89} & \textbf{85.5} & 
 \textbf{95.5} &\textbf{84.5} & \textbf{83.5} & \textbf{89.5} \\
w.opt.ICL    & 84          & 86.5   & 94.5       & 80.5          & 82            & 88            \\
w.opt.prompt & 60          & 84.5  & 93.0        & 77            & 79            & 84            \\
Plain          & 58          & 83     &93.5       & 77            & 81            & 84   \\ \bottomrule
\end{tabular}
\label{tab:ablation}
\end{table}

\subsection*{Computational Efficiency}
We demonstrate the computational efficiency gains of asynchronous distributed optimization through numerical results. As shown in Figure \ref{fig:compare_time}, AsynDBT and cenDBT are compared on their training times for 500 epochs across all datasets. The results indicate that, on all six datasets, it is evident that AsynDBT reduces training time by nearly 40\% while maintaining task performance comparable to cenDBT, which clearly highlights the superior computational efficiency of AsynDBT.

Besides, Figure \ref{fig:5g}, Figure \ref{fig:cola}, Figure \ref{fig:sst2}, Figure \ref{fig:mrpc}, Figure \ref{fig:qqp}, and Figure \ref{fig:qnli} illustrate the test accuracy and the cross-entropy loss curves over iterations for both AsynDBT and cenDBT on all datasets. In each figure, the solid blue line represents the mean performance of cenDBT, with the blue shading indicating its standard deviation. Similarly, the solid red line and shading correspond to the mean and standard deviation for AsynDBT.

Take Figure \ref{fig:5g} as an example, both cenDBT and AsynDBT gradually stabilize during training, with their cross-entropy loss converging to around 0.5. This convergence indicates that both algorithms reach a stable state where their outputs are consistent and reliable over time. Similar findings can be drawn from experimental results on other datasets.

\subsection*{Ablation Study}
\label{sec:ablation}
To evaluate the effect of optimizing in-context samples and prompt fragments on downstream task performance, we conduct ablation studies comparing accuracy across multiple datasets under different configurations:
(a) No optimization content (plain); (b) Only the prompt fragment optimized (w. opt. prompt); (c) Only the in-context samples optimized (w. opt. ICL).
In each setting, all optimized components are derived from the best results of the cenDBT algorithm for a specific run. Table \ref{tab:ablation} presents the numerical results. 

The findings show that optimizing the prompt fragment leads to a modest performance boost, whereas selecting appropriate demonstration markedly enhances performance and is key to achieving good results. Interestingly, on COLA dataset, cenDBT’s overall performance is slightly lower than that of the w. opt. ICL setting. Similarly, on QQP dataset, performance when optimizing the prompt fragment is actually lower than the plain setting. This may result from excessively long prompt fragments that negatively affect the LLM’s ability to process semantic information.

\section*{Conclusion}
In summary, this paper presents AsynDBT, a novel asynchronous distributed bilevel tuning framework designed to jointly optimize demonstration selection and prompt editing for efficient In-Context Learning (ICL). To the best of our knowledge, this work represents the first attempt to formulate ICL optimization as a bilevel black-box problem. Besides, we developed an efficient asynchronous distributed algorithm that effectively preserves data privacy and mitigates the "straggler issue" in heterogeneous environments. Furthermore, we provide a theoretical proof for the convergence of our proposed algorithm. The effectiveness and efficiency of our proposed AsynDBT are validated on six public benchmark datasets.




 
Despite its current strengths, AsynDBT remains limited by its reliance on static training datasets, which may lead to hallucination issues in domain-specific scenarios. Recent studies, such as the work on TrumorGPT\cite{Hang2025TrumorGPTGR}, have demonstrated that Graph-Based Retrieval-Augmented Generation (GraphRAG) can effectively mitigate the hallucination issues common in LLMs by leveraging updated semantic knowledge graphs. Inspired by these findings, our future work will explore the integration of a GraphRAG within our distributed bilevel optimization framework, specifically targeting high-stakes technical fields such as network operation and maintenance. Specifically, the upper-level objective will adaptively optimize retrieval parameters to ensure high-fidelity knowledge extraction from semantic graphs, while the lower-level task will perform ICL-based semantic reasoning grounded in retrieved factual triples.  This integration will empower our algorithm transcend the limitations of static data, allowing it to adapt to the rapid information flow of specialized domains while ensuring factual consistency and high-fidelity reasoning.

\section*{Acknowledgements}
\subsection*{Funding}
This work was supported by the  Tianchi Talents - Young Doctor Program (5105250183m),  Science and Technology Program of Xinjiang Uyghur Autonomous Region (2024B03028, 2025B04051), Regional Fund of the National Natural Science Foundation of China (202512120005) . 

\section*{Author contributions statement}
Hui Ma: Writing – Review and Editing, Conceptualization; Shaoyu Dou: Writing – Original Draft, Software, Methodology; Ya Liu: Experimental Analysis; Fei Xing: Supervision; Li Feng: Validation, Data curation; Feng Pi: Supervision. All authors reviewed the manuscript. 

\section*{Data Availability}
The datasets used and/or analyzed during the current study available from the corresponding author on reasonable request.

\section*{Declaration of competing interest}
The authors declare that they have no known competing financial interests or personal relationships that could have appeared to influence the work reported in this paper.

\bibliography{refs}
\end{document}